\documentclass[sigconf,authorversion,nonacm]{aamas} 

\usepackage{balance} 

\usepackage{natbib}
\usepackage{booktabs}
\usepackage{wrapfig}
\usepackage{graphicx}
\usepackage{amsmath}

\usepackage{amsthm}

\usepackage{amssymb}
\newtheorem{definition}{Definition}
\newtheorem{assumption}{Assumption}
\newtheorem{theorem}{Theorem}
\newtheorem{lemma}{Lemma}
\newtheorem{proposition}{Proposition}

\usepackage{microtype}
\usepackage{graphicx}
\usepackage{subfigure}
\usepackage{caption}
\usepackage{booktabs} 
\usepackage{amsfonts} 
\usepackage{algorithmic}
\usepackage{wrapfig}
\usepackage{multirow}
\usepackage{breqn}
\usepackage{todonotes}
\usepackage{lipsum}
\usepackage[ruled]{algorithm2e}
\usepackage{tabularx}
\makeatletter
\def\BState{\State\hskip-\ALG@thistlm}
\makeatother

\usepackage[utf8]{inputenc}
\usepackage[english]{babel}
\usepackage{bbm}

\usepackage{enumitem}
\newcommand{\SubItem}[1]{
    {\setlength\itemindent{15pt} \item[-] #1}
}

\acmSubmissionID{???}

\title{HOPE: Human-Centric Off-Policy Evaluation for \\ E-Learning and Healthcare}

\author{Ge Gao\footnotetext{Contact: \{ggao5, mchi\}@ncsu.edu}}
\affiliation{
  \institution{North Carolina State University}
  \city{Raleigh}
  \country{USA}}

\author{Song Ju}
\affiliation{
  \institution{North Carolina State University}
  \city{Raleigh}
  \country{USA}}

\author{Markel Sanz Ausin}
\affiliation{
  \institution{North Carolina State University}
  \city{Raleigh}
  \country{USA}}

\author{Min Chi}
\affiliation{
  \institution{North Carolina State University}
  \city{Raleigh}
  \country{USA}}

\begin{abstract}
Reinforcement learning (RL) has been extensively researched for enhancing human-environment interactions in various human-centric tasks, including e-learning  and health care. Since deploying and evaluating policies online are high-stakes in such tasks, off-policy evaluation (OPE) is crucial for inducing effective policies. In human-centric environments, however, OPE is challenging because the underlying state is often unobservable, while only aggregate rewards can be observed (students' test scores or whether a patient is released from the hospital eventually). In this work, we propose a \emph{human-centric OPE (HOPE)} to handle partial observability and aggregated rewards in such environments. Specifically, we reconstruct immediate rewards from the aggregated rewards considering partial observability to estimate expected total returns. We provide a theoretical bound for the proposed method, and we have conducted extensive experiments in real-world human-centric tasks, including sepsis treatments and an intelligent tutoring system. Our approach reliably predicts the returns of different policies and outperforms state-of-the-art benchmarks using both standard validation methods and human-centric significance tests.
\end{abstract}


\newcommand{\BibTeX}{\rm B\kern-.05em{\sc i\kern-.025em b}\kern-.08em\TeX}

\begin{document}


\pagestyle{fancy}
\fancyhead{}


\maketitle 


\section{Introduction}

Off-policy evaluation (OPE) sits at the epicenter of offline Reinforcement Learning (RL) research~\cite{doroudi2017importance,nachum2019dualdice,jiang2016doubly,thomas2016data}, which estimates the performance of an RL-induced policy leveraging prior knowledge obtained from historical data. OPE is especially important in human-centric environments where the execution of a bad policy can be costly and dangerous. For example, in healthcare, an overestimated policy could be ineffective and even increase mortality by wrong treatments. There are at least \textbf{two main challenges} for OPE to work with human-centric environments. One is that the underlying state of a human-centric environment is usually \textbf{\emph{unobservable}}~\cite{mandel2014offline,cassandra1998survey}, resulting in a partially observable Markov decision process (POMDP). Specifically, the observations may not contain sufficient information to fully reconstruct the underlying states. For example, the population considered in the environment could change over time, and their characteristics may be varied across different cohorts. In e-learning, student backgrounds can vary from semester to semester; similarly, in healthcare, patient demographics can change in hospitals across different regions~\cite{payne2017changing,khoshnevisan2021unifying}. \textbf{The other challenge} is that  \textbf{\emph{only an aggregated (or delayed) reward can be observed}} after a certain period of time, where all immediate rewards in between are often missing. The most appropriate rewards in e-learning and healthcare are student learning performance and patient outcomes, which are typically unavailable until the entire trajectory is complete. This is due to the complex natures of both learning and disease progression, which make it difficult to assess students' learning or patient health states moment by moment. More importantly, many instructional or medical interventions that boost short-term performance may not be effective over the long term. Different from delayed rewards in classic mouse-in-the-maze situations where agents receive insignificant rewards along the path and a significant reward in the final goal state (the food), in e-learning and healthcare, there are immediate rewards along the way but they are often \emph{unobservable}.
Prior OPE work~\cite{nachum2019dualdice,jiang2016doubly,abe2021off,gao2023variational} have achieved outstanding performance in simulated environments such as Mujoco~\cite{todorov2012mujoco}. However, one may not be able to directly apply such methods toward human-centric environments, since immediate rewards are missing and the environment is partially observable. A possible way is assuming the rewards are sparse that only indicate whether a task is completed partially or fully. Sparse rewards typically correspond to the attainment of some particular tasks such as a robot attaining designated waypoints which provide little feedback for immediate steps~\cite{rengarajan2022reinforcement}. In many human-centric tasks, interaction outcomes can be gradually improved step-by-step and the immediate reward on each step can be meaningful by itself. For example, students' cognitive outcomes and learning performance are gradually improved during interaction with the intelligent tutor in college~\cite{marwan2020adaptive}. It has been found that immediate rewards are more effective than sparse rewards, toward evaluating decision outcomes~\cite{azizsoltani2019unobserved}. Moreover, in human-centric environments, policy performance may be correlated with the horizon of the environment. For example, medical interventions that have shown good short-term performance may not be effective over the long-term~\cite{azizsoltani2019unobserved}. Consequently, it is important to reconstruct immediate rewards for OPE to work effectively in human-centric environments. 

On the other hand, prior evaluation metrics for OPE are generally error metrics (e.g., absolute error, rank correlations) proposed by~\citet{fu2021benchmarks,voloshin2019empirical}, while human-centric research often emphasizes the need for \emph{statistical significance test in empirical study}~\cite{guilford1950fundamental,zhou2022leveraging}. For example, rank correlation summarizes the performance of a set of policies' relative rankings using averaged returns. Statistical significance tests can tell how likely the relationship we have found is due only to random chance across samples, and they are commonly employed and easier to be interpreted by domain experts~\cite{guilford1950fundamental,ju2019importance}. The framework for validating OPE in human-centric environments may need to be extended beyond error metrics. 

In this work, we propose a \textbf{human-centric OPE (HOPE)} approach to tackle the two challenges above. Specifically, it first reconstructs immediate rewards from the aggregated rewards. Then, importance sampling is used to process the reconstructed rewards and estimate expected total returns. Any existing OPE method can be used jointly with the reconstructed rewards to estimate expected returns. We validate HOPE on two typical real-world human-centric tasks in healthcare and e-learning, \textit{i.e.} sepsis treatments and an intelligent tutoring system (ITS), and extend existing validation frameworks  with significance tests for human-centric environments. To summarize, our work has at least three main contributions:

    $\bullet$ To the best of our knowledge, HOPE is the first OPE approach that tackles both partial observability and aggregated rewards. We also provide theoretical bound for HOPE.
    
    $\bullet$ HOPE is extensively validated through real-world environments in real-world human-centric environments including e-learning and healthcare. The results show that our approach outperforms the state-of-the-art OPE approaches.
    
    $\bullet$ We introduce significance tests on top of existing OPE validation frameworks, to facilitate comprehensive study and comparison of OPE methods in human-centric environments. 

\section{Related Work}

In e-learning and healthcare, RL has been widely investigated to learn policies from historical user interaction data~\cite{nie2022data,gao2022gradient}. However, deploying and evaluating policies online are high stakes in such domains, as a poor policy can be fatal to humans in healthcare. It's thus crucial to propose effective OPE methods for human-centric environments. OPE is used to evaluate the performance of a \emph{target} policy given historical data drawn from (alternative) \emph{behavior} policies. Classic methods, such as expected cumulative reward (ECR)~\cite{tetreault2006comparing}, have been employed in real-world applications such as e-learning~\cite{chi2011empirically}. However, ECR is not statistically consistent, which is a significant concern for high-stakes domains~\cite{mandel2014offline}. In practice, researchers have found that selected policies based on ECR were even not effective compared to random policies with real students in terms of human-centric significant test~\cite{chi2011empirically,shen2016reinforcement}. 

Various contemporary OPE methods have been proposed and can be divided into three general categories~\cite{voloshin2019empirical}: 1) Inverse Propensity Scoring (IPS)~\cite{precup2000eligibility,doroudi2017importance}; 2) Direct Methods~\cite{jiang2016doubly,paduraru2013off,le2019batch,harutyunyan2016q,munos2016safe,farajtabar2018more,liu2018breaking,nachum2019dualdice,uehara2020minimax,xie2019towards,yang2022offline}; 3) Hybrid Methods~\cite{jiang2016doubly,thomas2016data}. IPS has been widely investigated in statistics~\cite{powell1966weighted,horvitz1952generalization} and RL~\cite{precup2000eligibility}, with the key idea to reweigh the rewards in historical data using the importance ratio between $\beta$ and $\pi$. IPS yields consistent estimates and it has several variations including IS~\cite{precup2000eligibility}, WIS~\cite{precup2000eligibility}, PDIS~\cite{precup2000eligibility}, and PHWIS~\cite{doroudi2017importance}, etc. Direct Methods directly estimate the value functions of the evaluation policy. For example, FQE~\cite{le2019batch} is functionally a policy evaluation counterpart to batch Q-learning. DualDICE~\cite{nachum2019dualdice} estimates the discounted stationary distribution ratios, which measure the likelihood that $\pi$ will experience the state-action pair normalized by the probability with which the state-action pair appears in the off-policy data. Hybrid Methods combine aspects of both IPS and direct methods. For example, DR~\cite{jiang2016doubly} is an unbiased estimator leveraging a direct method to decrease the variance of the unbiased estimates produced by IS. These OPEs have achieved desirable performance in  simulated environments. Recently, a few approaches have proposed OPE targeting some real-world tasks such as robotic grasping~\cite{irpan2019off,nie2021learning,chandak2021universal,gao2022offline}. The proposed methods generally assume the rewards are observable from the environment or sparse. It lacks investigation into the effectiveness of applying the existing OPE methods in real-world human-centric environments.

Recently, researchers have recognized the partial observability of human-centric environment and proposed OPE for POMDPs~\cite{tennenholtz2020off,bennett2021off,shi2022minimax,uehara2022future}. They assume that the underlying states may not exist and treat them as confounding for policy evaluation. For example, Tennenholtz et al. propose a method for POMDP with unobserved confounding and compare their method to IS in carefully-designed synthetic environments~\cite{tennenholtz2020off}. The results show that their proposed method can outperform IS under certain levels of confounding. In practice, such as in our intelligent tutoring experiment with real students, confounding is agnostic and unable to control, accompanied by missing immediate rewards. Our goal is to utilize limited observable information from real-world e-learning and healthcare environments, and effectively estimate policy performance with both partial observability and missing immediate rewards.



\begin{figure*}
    \centering
    \includegraphics[width=\linewidth]{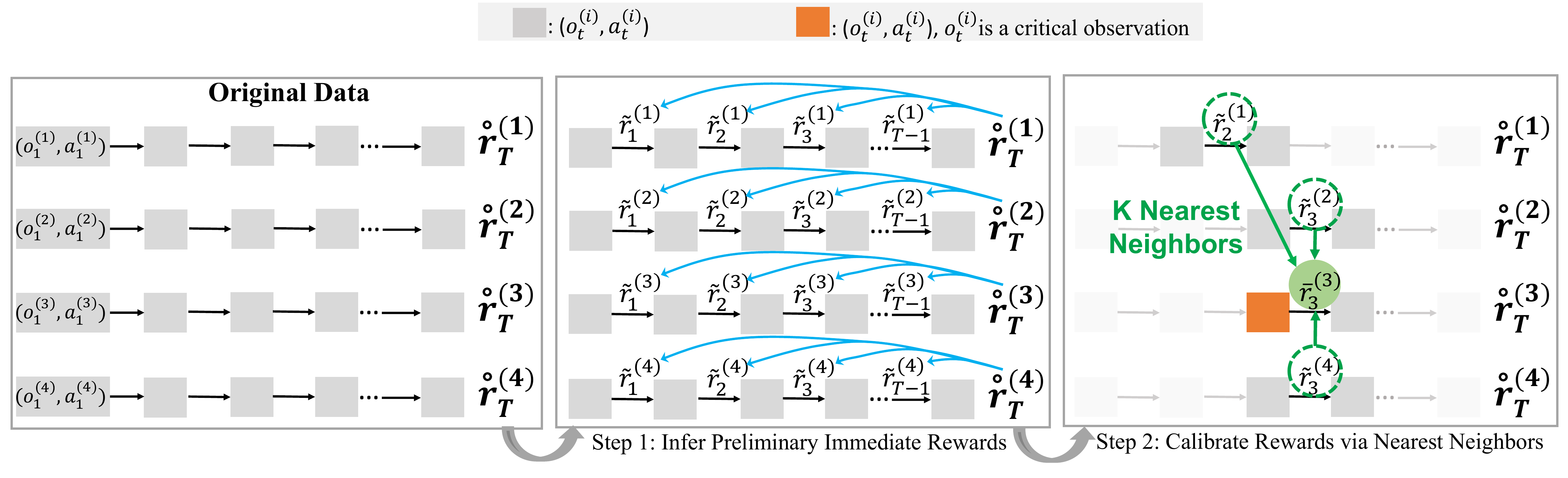}
    \caption{Illustration of reconstructing immediate rewards. First, the preliminary rewards are inferred from aggregated rewards. Then, the immediate rewards on critical observations are calibrated by averaging from nearest neighbors.}
    \label{fig:hope-illustrat}
    \vspace{-5pt}
\end{figure*}

\section{Human-Centric OPE (HOPE)}



\noindent\textbf{Problem Definition:} We consider the human-centric environment as a \textbf{partially observable Markov decision process (POMDP)}, which is a tuple $(\mathcal{S, A, O}, \mathcal{P}, \Omega, R, \gamma)$. 
The state space $\mathcal{S}$ is assumed \textit{unknown}. Moreover, $\mathcal{A}$ represents the action space, $O$ is the observation space, and $\mathcal{P}$ defines transition dynamics from the current state to the next states given an action. The observation model $p(o|s)\sim\Omega$ is also assumed \textit{unknown}. $R$ is the immediate reward function, and $r_t = R(s_t,a_t)$ are the \textit{immediate} rewards. $\gamma\in(0, 1]$ is discount factor. Each episode has a finite horizon $T$. In general, at time-step $t$, the agent is given an observation $o_t\in\mathcal{O}$ by the environment, then chooses an action $a_t\in\mathcal{A}$ following some policy $\pi:\mathcal{O}\rightarrow\mathcal{A}$. Then, the environment provides the immediate reward $r_t$ and next observation $o_{t+1}$. In human-centric tasks, it is common that the immediate rewards $r_t$ are unknown for most $t\in[1, T]$ in a trajectory. However, we assume that an aggregated reward $\mathring{r}_t = \sum_{i=t-W}^{t} \gamma^{i-t+W+1}r_i$ is issued by the environment every $W$ steps, quantifying the summed immediate rewards obtained between $(t-W)$-th and $t$-th steps. 


\noindent\textbf{Off-Policy Evaluation in POMDP:} the goal of OPE under POMDP is to estimate the expected total return over the \textit{target} policy $\pi$, $V^{\pi}=\mathbb{E}[\sum_{t=1}^T{\gamma^{t-1}r_t}|a_t\sim\pi]$, using historical data collected over a \textit{behavioral} policy $\beta\neq\pi$ deployed to the environment. Specifically, the historical data $\mathcal{D}=\{..., [..., (o_t,a_t,r_t,o'_t), ...]^{(i)}, ...\}^N_{i=1}$ consist of a set of $N$ trajectories, where each trajectory is denoted as~$\tau^{(i)}$. 



\vspace{0.1in}
\noindent Figure~\ref{fig:hope-illustrat} illustrates how HOPE addresses the two major challenges of OPE in human-centric environments: partial observability and aggregated rewards.
Specifically, it first infers preliminary immediate rewards from aggregated rewards (Step 1) , and reconstructs immediate rewards $r_t$ using nearest neighbors from trajectories $\mathcal{D}$ (Step 2), which takes into account the uncertainty introduced by partial observability. The intuition of Step 2 follows findings presented in previous research~\cite{yang2020student,keramati2022identification,gao2021early}, where humans with similar underlying states can present similar behaviors. Then, we use a general method, weighted importance sampling, to process the reconstructed rewards and estimate the expected total return $V^\pi$, since it is the most straightforward approach and can help isolate the source of improvements brought in by the immediate rewards reconstruction framework we propose. 

\subsection{Reconstruct Immediate Rewards via Nearest Neighbors}
\label{sec:KNN}

In human-centric environments, immediate rewards are generally not available, while the aggregated rewards themselves may not provide full information that can be directly used to estimate the expected total return of a target policy. 
Although prior works has noted that estimating immediate rewards can provide useful information toward training RL policies~\cite{mataric1994reward,azizsoltani2019unobserved}, it has not be extensively investigated in the context of OPE; since prior OPE works generally assume immediate rewards are observable from the environment or sparse. In many human-centric tasks, interaction outcomes can be gradually improved step-by-step and each immediate
reward on each step can represent meaningful information by itself. Therefore, we propose to reconstruct immediate rewards from historical data, aiming to enrich the information provided for OPE. 



In the OPE problem we consider, the underlying state is unobservable due to POMDP. And the aggregated rewards are assumed to be obtained at the end of an episode following $\mathring{r}=\sum_{t=1}^T{\gamma^t r_t}$, with $T$ being the horizon of the environment. In prior work, function approximation can be used to infer immediate rewards $r_t$ from the aggregated rewards $\mathring{r}$~\cite{azizsoltani2019unobserved}. Consequently, we first produce rough per-step \textit{preliminarily} reconstructed reward $\tilde r$ taking as inputs observations and actions, $\tilde r = f(o_t,a_t|\theta)$, and trained toward,~\textit{i.e.}, 
\begin{align}
\label{hope:loss}
    \min_\theta l(\theta) = \frac{1}{N}\sum\limits_{i=1}^N \big(\mathring{r}-\sum_{t=1}^T{\gamma^{t-1} \tilde r}\big)^2.
\end{align}
We keep the objective straightforward and use a standard training method in~\cite{ausin2021infernet}, so that the source of performance improvements can be isolated. For conciseness, we refer to $\tilde r$ as \textit{preliminary immediate rewards} throughout the rest of the paper. In our experiments described in Section~\ref{sec:ablation}, we compare it to treating the rewards as sparse, and the results present the effectiveness of the preliminary immediate rewards reconstruction procedure. 


Consider that the mapping $f:\mathcal{O}\times\mathcal{A}\rightarrow \mathbb{R}$ learned above may not perfectly reconstruct the immediate reward function $R:\mathcal{S}\times\mathcal{A}\rightarrow \mathbb{R}$, since the observation model $p(o|s)$ and state space $\mathcal{S}$ are both assumed unknown. Previous research have found that humans with similar underlying state can be observed similar behaviors~\cite{yang2020student,yin2020identifying}, and some observations may provide critical information over others~\cite{torrey2013teaching}. We then follow such intuitions and introduce a remedy that uses $\tilde r_t$ and its nearest neighbours, from trajectories that are similar to each other, to reconstruct immediate rewards $r_t$ more accurately. 

Specifically, to capture the information pertaining to the underlying states, we define the reconstructed immediate reward $\hat{r}_t$ considering both average rewards associated with neighboring trajectories and observations that may provide more critical information over others, by which we call \textit{critical observations}, i.e.,
\begin{equation}
\label{eq:potential_reward}
    \hat{r}_t=\mathbbm{1}(o_t\in O^*)\bar{r}_t+\mathbbm{1}(o_t\notin O^*)\tilde{r}_t,
\end{equation}
where $O^*$ is the set of critical observations, $\mathbbm{1}(\cdot)$ is an indicator function.

We first introduce how to identify the critical observations. Given that $Q$-functions, $Q^\pi(o, a)$, representing the expected return of taking action $a$ at observation $o$, the difference of $Q^{\pi}$ between any two actions (over the same $o$) can be used to quantify the magnitude of the difference in the final outcomes. 
As a result, we choose to leverage such $Q$-difference toward determining critical observations for OPE. 
A formal definition of a critical observation can be found below.  
\begin{definition}[Critical Observation]
\label{def:potential}
In a discrete action POMDP, observation $o$ is a critical observation if there exists a constant $h$ such that the maximum difference of $Q^\pi(o,a), a\in\mathcal{A}$ is greater than $h$, i.e., 
\begin{align}
\label{eqa:potential-state}
    \max_{a',a''\in\mathcal{A}}(Q^\pi(o,a')-Q^\pi(o,a''))>h. 
\end{align}
\end{definition}
In Fig.~\ref{fig:critic-obs}, we illustrate how critical observations are identified using the threshold $h$. The need of critical observations is further justified by ablation study in Section~\ref{sec:ablation}.

Denote $\hat{r}^{(i)}_t$ as the reconstructed immediate reward at time-step $t$ on trajectory $\tau^{(i)}$. We define the averaged reward $\bar{r}^{(i)}_t$ associated with critical observation $o^{(i)}_t$
as an average of preliminary immediate rewards $\tilde{r}$ that occur at neighboring events following
\begin{equation}
    \bar{r}^{(i)}_t = \frac{1}{K} \sum_{k}{\tilde{r}^{(k)}_{t'}},
    \label{eqa:knn}
\end{equation}
where the summation is performed over all $\mathring{r}^{(k)}_{t'}$ and 
$K$ is the total number of nearest neighbours pertaining to the algorithm~\ref{alg:KNN}. Specifically, we define that two trajectories are neighbors if they have similar visitation distributions over the observation and action spaces, following
\begin{equation}
    d(\tau^{(i)},\tau^{(k)}) = similarity(o^{(i)},o^{(k)})+ similarity(a^{(i)},a^{(k)})
    \label{eqa:distance}
\end{equation}
where the similarity of distributions can be calculated by some measures such as Kullback-Leibler (KL) divergence~\cite{kullback1951KL}.
Then we define that two observations are neighbors if they are from two neighboring trajectories, and share the most similar observations, since the moment that sharing similar observations from the similar trajectories may represent close underlying state of humans~\cite{yang2020student,yin2020identifying}. To calculate similarity between two observations, one can use distance measures such as Euclidean. We can always find one and only one neighboring observation on a given neighboring trajectory, by breaking ties appropriately~--~such as selecting the earliest observation. Algorithm~\ref{alg:KNN} summarizes how to find the K-nearest neighbors of $o_t^{(i)}$. In our experiment, the proposed distance to find nearest neighbors is compared to random selection in Section~\ref{sec:ablation}, and the results show the superior performance of the proposed distance.

\begin{figure}
\centering
\includegraphics[width=0.9\linewidth]{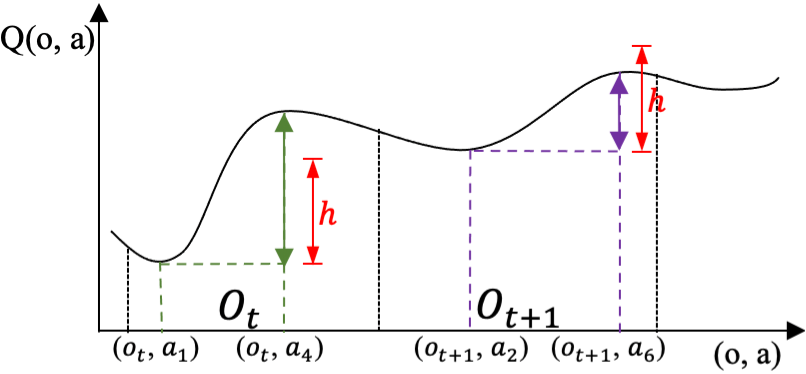}
\caption{Identify critical observations based on the maximum difference of $Q$ values. $o_t$ is a critical observation with difference of $Q$ greater than a threshold $h$.}
\label{fig:critic-obs}
\end{figure}

\begin{algorithm}[t]
\caption{Find K Nearest Neighbors of $o^{(i)}_t$.}\label{alg:KNN}
\begin{algorithmic}[1]
\REQUIRE Historical data $\mathcal{D}$, hyper-parameter $K$.\\
\ENSURE
\STATE Initialize an empty set $Z^i_t$. 
\FOR {each trajectory $\tau^{(j)}\in \mathcal{D}$}
\STATE Calculate the distance $d(\tau^{(i)},\tau^{(j)})$ following~\eqref{eqa:distance}.
\ENDFOR
\STATE Arrange $d(\tau^{(i)},\tau^{(j)}), j\in N$ in non-decreasing order, take the first $K$ elements and find the corresponding trajectories as \textit{neighboring trajectories of $\tau^{(i)}$}. 
\FOR {each neighboring trajectory $\tau^{(k)}$ of $\tau^{(i)}$}
\FOR {each observation $o^{(k)}_{t'}\in\tau^{(k)}$}
\STATE Calculate the distance $l^{ik}_{tt'}$ between $o^{(k)}_{t'}$ and $o^{(i)}_{t}$.
\ENDFOR
\STATE Arrange $l^{ik}_{tt'}, t'\in T$ in non-decreasing order, take the first element and find the corresponding observation $o^{(k)}_{t'}$ as \textit{neighboring observation of $o^{(i)}_t$}.
\STATE $Z^i_t\leftarrow Z^i_t\cup o^{(k)}_{t'}$
\ENDFOR
\STATE Return the set $Z^i_t$.
\end{algorithmic}
\end{algorithm}

\subsection{HOPE}
Based on the proposed solutions regarding the challenges within human-centric environment, we propose our \textit{human-centric OPE (HOPE)} by incorporating estimated immediate rewards with weighted importance sampling (WIS). Let $w_i:=\Pi_{t=1}^{T} {\pi(a^{(i)}_t|o^{(i)}_t) / \beta(a^{(i)}_t|o^{(i)}_t)}$ be an \textit{importance weight}, which is the probability of the first $T$ steps of $\tau^{(i)}$ under the \textit{target policy}, $\pi$, divided by its probability under the behavior policy, $\beta$~\cite{precup2000eligibility}. The estimated return of HOPE for an target policy $\pi$ is 
\begin{equation}
    V^{\pi}(\mbox{HOPE}) = \frac{1}{\sum_{i=1}^N{w_i}} \sum_{i=1}^N{w_i} \sum_{t=1}^{T}
    {\gamma^{t-1} 
    \hat{r}^{(i)}_t},
\label{eqa:npr}
\end{equation}
We choose to use WIS here since it is the most straightforward approach. Moreover, this can help isolate the source of improvements brought in by immediate rewards reconstruction framework we propose. Prior work note that the lower variance of WIS may produce a larger reduction in expected square error than the additional error incurred due to the bias compared to some unbiased OPE such as importance sampling (IS) in practice~\cite{thomas2015policy}. Our real-world experimental results in Section~\ref{sec:experiment} further support this.

In general, the importance weight assumes that the support of the evaluation policy $\pi$ is a subset of the behavior policy $\beta$, which is enforced by Assumption~\ref{assp:1}~\cite{thomas2015safe}:

\begin{assumption}
\label{assp:1}
If $\pi(a|o)\neq 0$, then $\beta(a|o)\neq 0$, where $a\in\mathcal{A}, o\in\mathcal{O}$.
\end{assumption}

\paragraph{Upper and Lower Bounds of HOPE} We define $\delta^{ik}$ as the event where $\tau^{(k)}$ is a neighbor of $\tau^{(i)}$. 
Also, we define the counts of neighboring events as $K=\sum_{k=1}^N \mathbbm{1}(\delta^{ik})$, where $\mathbbm{1}(\cdot)$ is the indicator function. We then construct the nearest-neighbors matrix:
    $M^{ik} = \frac{\mathbbm{1}(\delta^{ik})}{K}$.
As The $N\times N$ matrix $M^{ik}$ can be computed from the data and be used to compute the estimated immediate rewards for all observation-action pairs using the following proposition.
\begin{proposition}
\label{prop:knn-reward}
For all transitions in the data, the estimated immediate rewards for corresponding observation-action pairs are given by
\begin{equation}
    \hat{r}^{(i)}_t = \sum_{k}M^{ik} \equiv [\textbf{Mu}]_i
\end{equation}
\end{proposition}

\begin{proof}
A given reward on trajectory $i$ at timestamp $t$, that averaging over all rewards on trajectory $j$ at timestamp $t'$ such that $\delta_{tt'}^{ij}$ holds, can be written as $\frac{1}{K}\sum_{(j,t'):\delta_{tt'}^{ij}} = \sum_{j,t'}\frac{\mathbbm{1}(\delta_{tt'}^{ij})}{K} = \sum_{j,t'}M_{tt'}^{ij}$.
Therefore, assume $u(o,a)$ is a function over the observation-action space and $\textbf{u}$ is the vector containing the quantity $u_i=u(o^{(i)},a^{(i)})$ for every $(o^{(i)},a^{(i)})$, the nearest-neighbor estimation of $u(o^{(i)},a^{(i)})$ is given by $[\textbf{Mu}]_i$.
\end{proof} 

In the case of $\tilde{r}\in[\tilde{r}_{lb},\tilde{r}_{ub}]$, $\hat{r}\in[\tilde{r}_{lb},\tilde{r}_{ub}]$ according to Proposition~\ref{prop:knn-reward}. Denote the returns of trajectory $\tau^{(i)}$ as $G(\tau^{(i)})$. Following~\cite{thomas2015safe}, we also write $G(\tau^{(i)})\in [0, 1]$ s.t. $G(\tau^{(i)}):=\frac{(\sum_{t=1}^T\gamma^{t-1}\hat{r}^{(i)}_t)-G_{lb}}{G_{ub}-G_{lb}}$ to denote quantification of how good a trajectory $\tau^{(i)}$ is. Then the HOPE estimation is written by $\frac{1}{\sum_{n=1}^N{w_i}} \sum_{i=1}^N{w_i} G(\tau^{(i)}) = \frac{\sum_{i=1}^N{w_i} G(\tau^{(i)})}{\sum_{i=1}^N{w_i}}$. Then the upper and lower bound of HOPE estimation can be calculated using the following lemma, the proof of which is given in~\cite{thomas2015safe}.

\begin{lemma}
\label{lemma:1}
Let $\pi$ and $\beta$ be any policy such that Assumption~\ref{assp:1} holds, then for any constant integer $m\geq 1$, $\mathbb{E}[\prod^m_{t=1} \frac{\pi(a_t|o_t)}{\beta(a_t|o_t)} | \tau\sim \beta] = 1.$
\end{lemma}

From Lemma~\ref{lemma:1}, with the number of samples increasing, the denominator of HOPE tends towards $n$. HOPE estimator is bounded within $[0,1]$, and so $\rho^{HOPE}_{lb}(\pi,\beta)=0$ and $\rho^{HOPE}_{ub}(\pi,\beta)=1$ for all $\pi$ and $\beta$, where $\rho(\pi)=\mathbb{E}[G(\tau)|\tau\sim\beta]$.

\paragraph{Consistency of HOPE} We also show that HOPE is a consistent estimator of $\rho(\pi)$ if there is a single behavior policy (\textit{Theorem~\ref{theorem:single}}) or if there are multiple behavior policies that satisfy a technical requirement (\textit{Theorem~\ref{theorem:multiple}}), following work by~\cite{thomas2015safe}.

\begin{theorem}
\label{theorem:single}
If Assumption~\ref{assp:1} holds and there is only one behavior policy, then HOPE is a consistent estimator of $\rho(\pi)$.
\end{theorem}
\begin{proof}
When there is only one behavior policy, HOPE estimation can be rewrote as $\frac{\frac{1}{N}\sum_{i=1}^N{w_i} G(\tau^{(i)})} {\frac{1}{N} \sum_{i=1}^N{w_i}}$ by multiplying both its numerator and denominator by $\frac{1}{N}$. Then the numerator is equal to $IS^{\pi}$, which is a consistent estimator of $\rho(\pi)$ as proved in prior work~\cite{precup2000eligibility,thomas2015safe}), thus the numerator converges almost surely to $\rho(\pi)$. For the denominator, by Lemma~\ref{lemma:1}, we have that
\begin{equation}
\label{eq:appendix-prod}
    \mathbb{E}[\prod^T_{t=1} \frac{\pi(a^{(i)}_t|o^{(i)}_t)}{\beta(a^{(i)}_t|o^{(i)}_t})] = 1, \text{for all }i\in[{1,\dots,N}].
\end{equation}
Each term $\prod^T_{t=1} \frac{\pi(a^{(i)}_t|o^{(i)}_t)}{\beta(a^{(i)}_t|o^{(i)}_t)}$ is identically distributed for each $i\in{[1,\dots,N]}$, as there is only one behavior policy. By the Khintchine strong law of large numbers~\cite{sen1994large}, we have that
\begin{equation}
\label{eq:appendix-as}
    \frac{1}{N} \sum^N_{i=1} \prod^T_{t=1} \frac{\pi(a^{(i)}_t|o^{(i)}_t)}{\beta(a^{(i)}_t|o^{(i)}_t)} \xrightarrow{\text{a.s.}} 1.
\end{equation}
By the property of almost sure convergence~\cite{jiang2010large}, 
HOPE converges almost surely to $\rho(\pi)$, and so HOPE is a consistent estimator of $\rho(\pi)$. 
\end{proof}

We also provide the proof for the consistency of HOPE when there are multiple behavior policies. 
\begin{theorem}
\label{theorem:multiple}
If Assumption~\ref{assp:1} holds and there exists a constant $\epsilon>0$ such that
$\beta_i(a|o)\geq \epsilon$ for all $i\in\{1,\dots, N\}$ and $(a|o)$ where $\pi(a|o)\neq 0$, then HOPE is a consistent estimator of $\rho(\pi)$ if there are multiple behavior policies.
\end{theorem}
\begin{proof}
When there are multiple behavior policies, similar to the proof of Theorem~\ref{theorem:single}, the numerator is equal to $IS^{\pi}$. The numerator converges almost surely to $\rho(\pi)$. For the denominator, by Lemma~\ref{lemma:1} we have that Equation~\ref{eq:appendix-prod} holds. Each term $\prod^T_{t=1} \frac{\pi(a^{(i)}_t)|o^{(i)}_t}{\beta_i(a^{(i)}_t)|o^{(i)}_t} \in [0,\frac{1}{\epsilon^T}]$ and therefore has bounded variance. By the Kolmogorov strong law of large numbers~\cite{sen1994large}, we have that almost surely convergence~\ref{eq:appendix-as} holds. Therefore, HOPE is a
consistent estimator of $\rho(\pi)$ if there are multiple behavior policies.
\end{proof}

\begin{figure} [b!]
\centering
  \centering
  \includegraphics[width=0.75\linewidth]{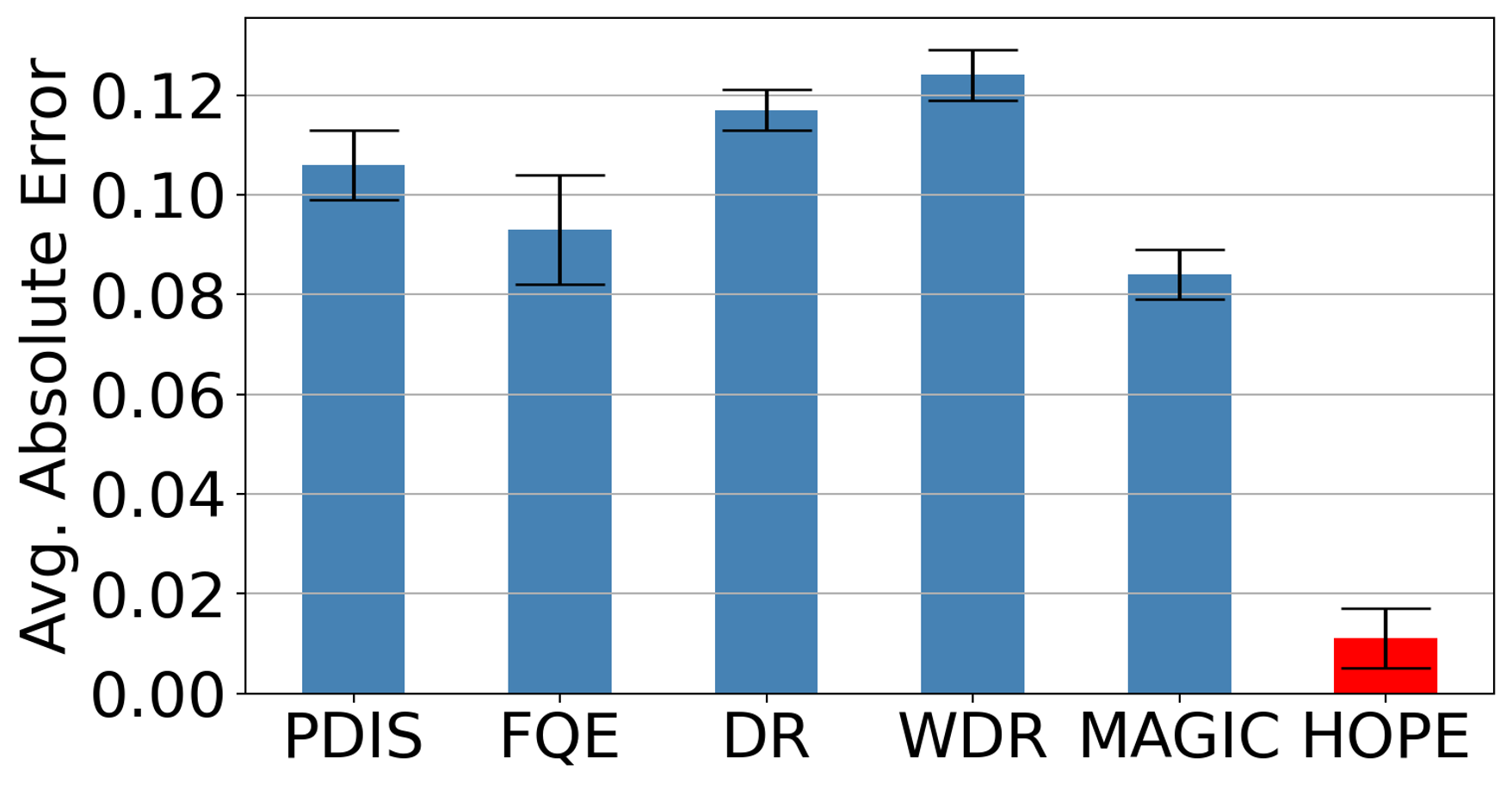}
  \caption{Average absolute error and standard deviation from synthetic sepsis environment ($\gamma$=0.99). IS, WIS, PHWIS, DualDICE are unable to select the best policy.}
  \label{fig:sim-result}
\end{figure}

\begin{figure*}[t!]
  \centering
  \includegraphics[scale=0.7]{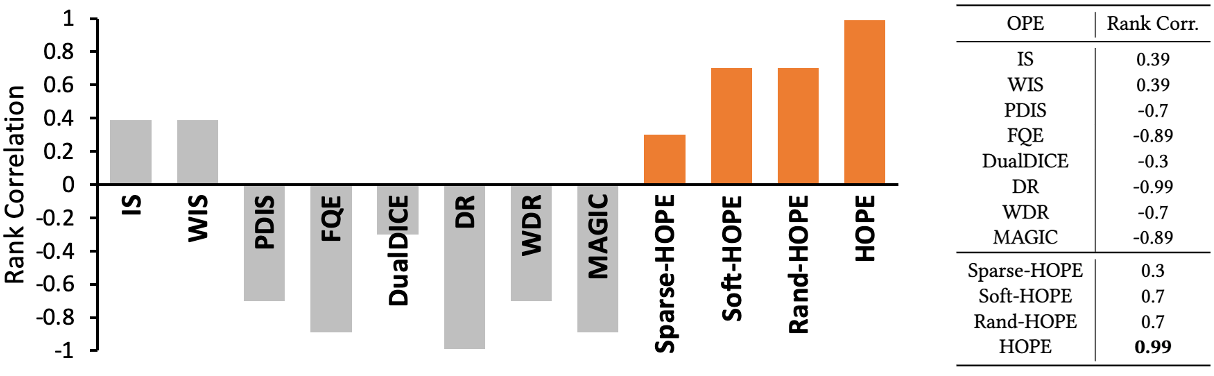}
  \caption{Spearman's rank correlation coefficient across policies from real-world medical system ($\gamma$=0.99). PHWIS is unable to produce a meaningful result probably due to the highly varied lengths of trajectories.}
  \label{fig:health-result}
\end{figure*} 

\section{Experiments}
\label{sec:experiment}
We conduct experiments on two real-world human-centric tasks, sepsis treatment, and intelligent tutoring, to validate our proposed approach. Since our focus is the human-centric environments, we don't conduct experiments on control tasks such as D4RL~\cite{fu2020d4rl}.

\subsection{Benchmarks}
We use \textbf{\emph{nine}} state-of-the-art benchmarks, which cover a variety of approaches that have been explored for OPE: \textit{Four IPS methods}, including IS~\cite{precup2000eligibility}, WIS~\cite{precup2000eligibility}, Per-Decision IS ( PDIS)~\cite{precup2000eligibility}, and Per-Horizon WIS (PHWIS)~\cite{doroudi2017importance}. For PHWIS, we follow the PHWIS-Behavior as in~\cite{doroudi2017importance}, as we assume the lengths of the trajectories do not depend on the policy that is used to generate them. \textit{Two direct methods}, Fitted Q Evaluation (FQE)~\cite{le2019batch} and Dual stationary DIstribution Correction Estimation (DualDICE)~\cite{nachum2019dualdice}. For FQE, as in~\cite{le2019batch}, we train a neural network to estimate the value of the evaluation policy $\pi_e$ by bootstrapping from $Q(o',a')$. For DualDICE, we use the open-sourced code in its original paper. \textit{Three hybrid methods}, including Doubly Robust (DR), Weighted DR (WDR)~\cite{thomas2016data}, and MAGIC~\cite{thomas2016data}, which finds an optimal linear combination among a set that varies the switch point between WDR and direct methods. For MAGIC, we use the implementation of~\cite{voloshin2019empirical}. 

\subsection{Validating OPE Performance}
In this work, we use two types of procedures to validate the performance of OPE methods. We use standard metrics including absolute error, regret@1, and Spearman's rank correlation coefficient~\cite{spearman1987proof} that are commonly used in prior OPE approaches. Definitions of these metrics are described in Appendix~\ref{appendix:experiment-detail}. Moreover, we use the human-centric significance test to measure \emph{the statistical significance} between the OPE-estimated returns across different policies. For each evaluation policy, we use bootstrapping by episodes as introduced in~\cite{hao2021bootstrapping}.
For e-learning, we also empirically evaluated the induced policies. As for many human-centric tasks, one key measurement for the RL-induced policy is whether they significantly outperform the current expert policy~\cite{zhou2020improving}. Therefore, we conduct a t-test over OPE estimations obtained from bootstrapping. It measures whether there is a significant difference between the mean value of OPE estimations on one policy against another.


\subsection{Sepsis Treatment}
Sepsis, defined as life-threatening organ dysfunction in response to infection, is the leading cause of mortality and the most expensive condition associated with in-hospital stay, accounting for more than \$24 billion in annual costs in the United States \cite{liu2014hospital}. In particular, septic shock, which is the most advanced complication of sepsis due to severe abnormalities of circulation and/or cellular metabolism \cite{bone1992definitions}, reaches a mortality rate as high as 50\% \cite{martin2003epidemiology}. 
Sepsis treatment is a highly challenging problem and has raised tremendous investigation~\cite{gao2022reconstructing,gao2022reinforcement}.

\subsubsection{Synthetic Sepsis Environment}
We use Oberst and Sontag’s sepsis model~\cite{oberst2019counterfactual} in the management of sepsis in ICU patients. Following the settings from~\cite{namkoong2020off}, the discrete observation space consists of a binary indicator for diabetes, and four vital signs (heart rate, blood pressure, oxygen concentration, glucose level) that take values in a subset of \texttt{\{very\_high, high, normal, low, very\_low\}}. The simulated environment contains a total of 1440 observations, and 8 actions characterized by assigning a binary value (0 or 1) toward each option in \texttt{\{antibiotics, vasopressors, mechanical\_ventilation\}}. The simulation continues either until at most T = 5 (horizon) time steps (0 rewards), death (-1 reward), or discharge (+1 reward). Patients are discharged when all vital signs are in the normal range without treatment. Patients die if at least three vitals are out of the normal range. We use the behavior policy and three evaluation policies following~\cite{namkoong2020off}. Figure~\ref{fig:sim-result} shows that HOPE performs the best in terms of average absolute error across evaluation policies. Detailed results are presented in Appendix~\ref{appendix:sim}.



\subsubsection{Real-World Medical System}
\label{subsubsec:mayo}
A medical system is another important domain for OPE, considering the safety of treating patients. 
In our experiment, we use Electronic Health Records (EHRs) collected from a large hospital in the United States with overall 221,700 visits patients over two years.
The observation space consists of 15 continuous sepsis-related clinical attributes, including seven vital signs \texttt{\{HeartRate, RespiratoryRate, PulseOx, SystolicBP, DiastolicBP, MAP, Temperature\}} and eight lab analytes \texttt{\{Bands, BUN, Lactate, Platelet, Creatinine, BiliRubin, WBC, FIO2\}}. The size of action space is 4 with two binary treatment options over \texttt{\{antibiotic\_administration, oxygen\_assistance\}}. Four stages of sepsis are defined by the clinicians, and the rewards are set for each stage: infection ($\pm$5), inflammation ($\pm$10), organ failure ($\pm$20), and septic shock ($\pm$50). The designated negative rewards are given when a patient enters the corresponding stage and positive rewards are given back when the patient recovers from the stage. The collected trajectories' lengths range from 1 to 1160. We assume that the clinical care team is well-trained with medical knowledge and follows standard protocols in sepsis treatments, thus we learn the expert policy as introduced in~\cite{azizsoltani2019unobserved}. We train policies using Deep Q Network (DQN)~\cite{mnih2015human} with varied hyperparameters and select five as evaluation policies as discussed in Appendix~\ref{appendix:health}. As prior work in sepsis research~\cite{komorowski2018artificial,azizsoltani2019unobserved,raghu2017deep} identifies septic shock rate as an important 
criterion for learning policies, we calculate Spearman's rank correlation coefficient between the policies' rank using estimated values given by OPE and the actual policies' rank in terms of septic shock rates. Experimental details are provided in Appendix~\ref{appendix:health}.

Figure~\ref{fig:health-result} shows the results of HOPE and benchmarks. The grey-shaded columns represent the benchmark results, and the orange-shaded columns represent the results from HOPE and its variations. Overall, HOPE performs the best in terms of rank correlation. Interestingly, we notice that IS and WIS outperform other benchmarks, while they can be suffering from long-horizon in prior theoretical work~\cite{liu2018breaking}. A possible reason is that both methods benefit more from the reduction in expected error than the variance incurred due to horizon under our real-world settings. Similar findings are reported in some long-horizon environments~\cite{fu2021benchmarks}.        


\begin{figure}[]
    \centering
    \includegraphics[width=\linewidth]{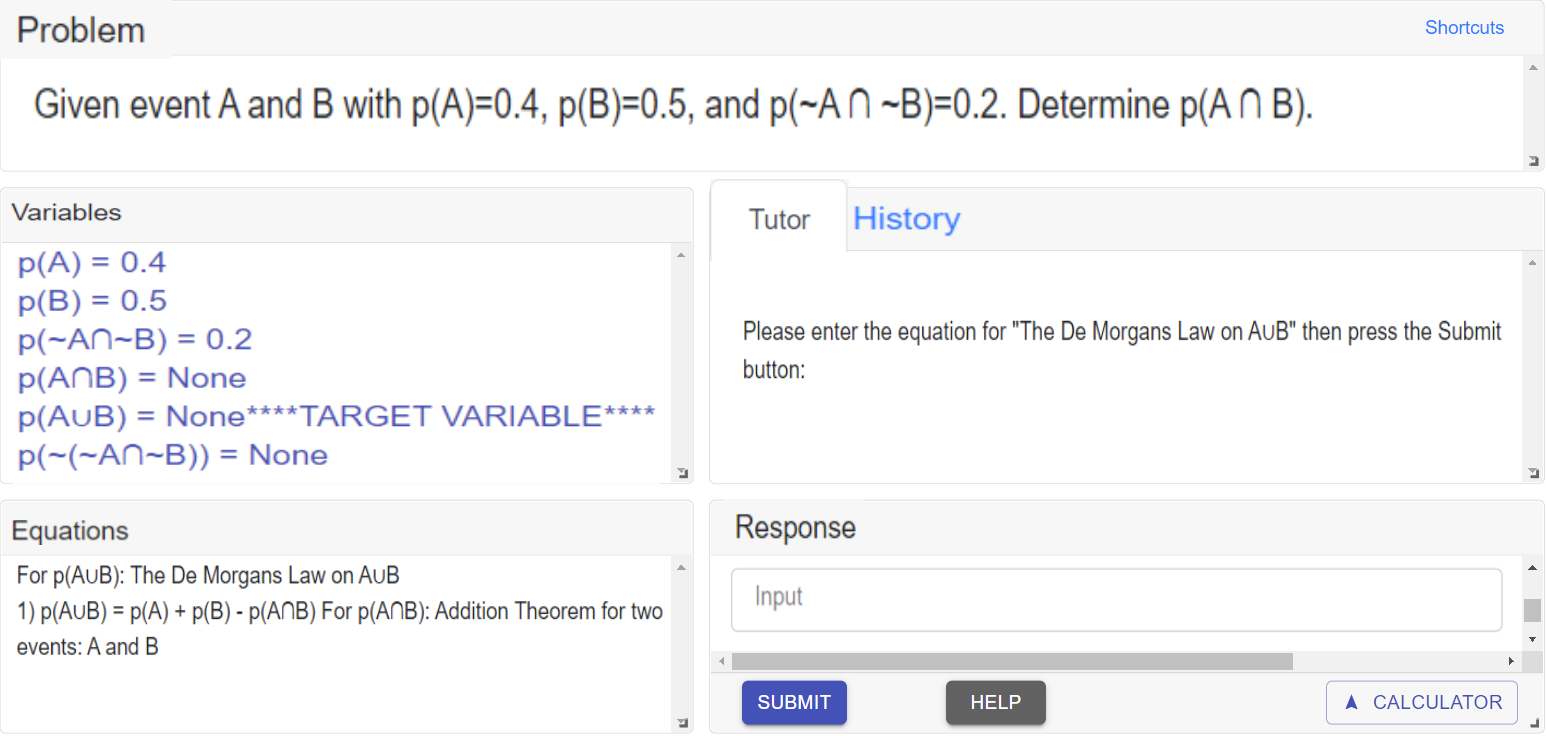}
    \caption{Our ITS GUI. The problem statement window (top) presents the statement of the problem. The dialog window (middle right) shows the message the tutor provides to the students. Responses, e.g., writing an equation, are entered in the response window (bottom right). Any variables and equations generated through this process are shown on the variable window (middle left) and equation window (bottom left).}
    \label{fig:pyrenees}
    \vspace{-10 pt}
\end{figure}

\subsection{Real-World Intelligent Tutor}

\begin{figure*}
  \centering
  \includegraphics[scale=0.35]{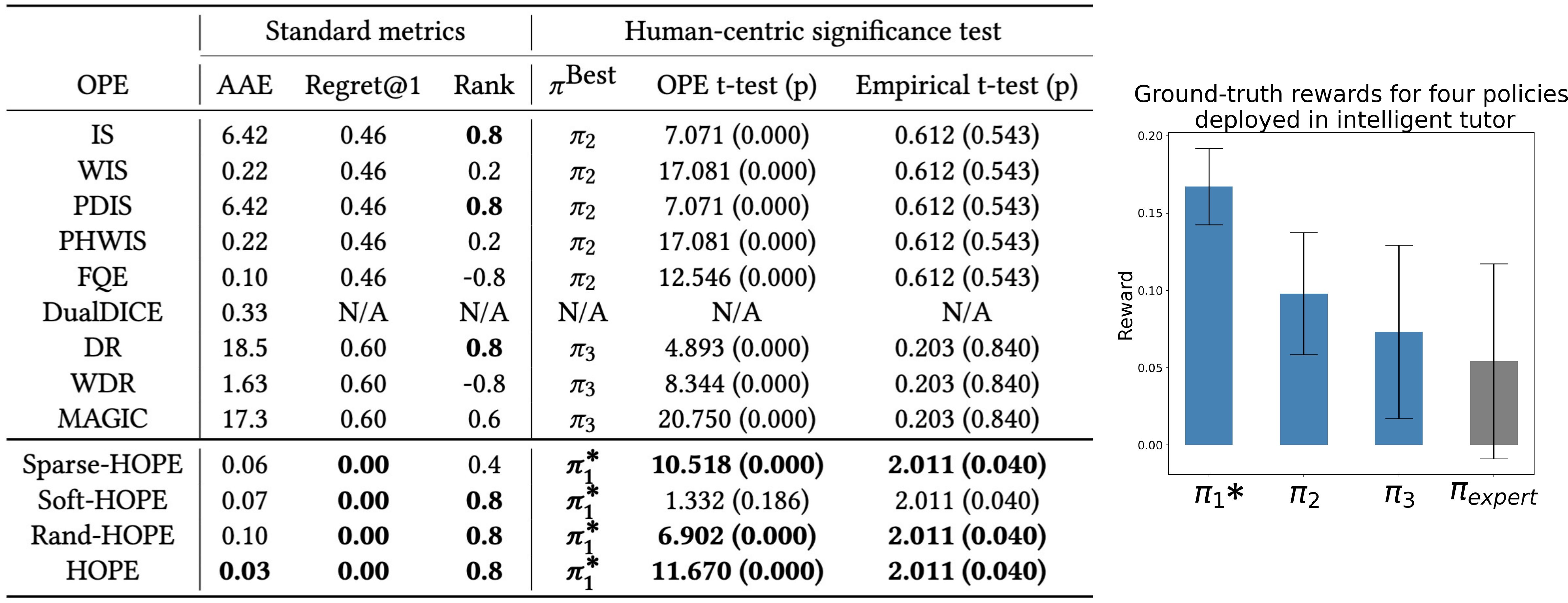}
  \caption{Validating OPE performance (left) and empirical results (right) from real-world intelligent tutoring ($\gamma=0.9$). Left: Standard metrics include average absolute error (AAE), regret@1, and Spearman's rank correlation coefficient (Rank). The best results on each metric are bolded. Human-centric significance test results include the best policy selected by OPE, offline significance test with bootstrapping on the best policy selected by OPE and expert policy, and empirical significance test on these two policies, at the level of $p<0.05$. The OPE significance test that aligns with the empirical test is bolded. DualDICE estimates the performance of all policies equally, thus its rank-related results are unavailable. Right: Ground-truth rewards for four policies. $\pi_1$ receives the highest average reward and is the only policy that differs significantly from $\pi_{expert}$ ($p=0.04$), as indicated by the asterisk.}
  \label{tab:edu-result}
\end{figure*}

Intelligent Tutoring Systems (ITSs) are computer systems that mimic aspects of human tutors and have also been shown to be successful~\cite{DBLP:journals/aiedu/VanLehn06,abdelshiheed2020metacognition,abdelshiheed2021preparing}. They aim to provide instruction or feedback to support students' learning, which is an important application of RL to improve students' engagement and learning outcomes. We use a web-based ITS which teaches computer science students probability knowledge, covering ten major principles such as the complement theorem. Students' interaction logs are collected over seven semesters of classroom studies (including 1,307 students) in an undergraduate computer science course at a large public university in the United States. Figure~\ref{fig:pyrenees} shows the GUI of the tutor.


During tutoring, there are many factors that might determine or indicate students’ learning state, but many of them are not well understood by educators. Thus, to be conservative, we extract varieties of attributes that might determine or indicate student learning observations from student-system interaction logs. In sum, 142 attributes with both discrete and continuous values are extracted, which can be categorized into the following five groups: 

(i) \textbf{Autonomy (10 features)}: the amount of work done by the student, such as the number of times the student restarted a problem; 

(ii) \textbf{Temporal Situation (29 features)}: the time-related information about the work process, such as average time per step; 

(iii) \textbf{Problem-Solving (35 features)}: information about the current problem-solving context, such as problem difficulty;

(iv) \textbf{ Performance (57 features)}: information about the student’s performance during problem-solving, such as percentage of correct entries;

(v) \textbf{Hints (11 features)}: information about the student’s hint usage,
such as the total number of hints requested.

The agent will make 10 decisions: for each problem, the agent will decide 
whether the student should \emph{solve} the next problem, \emph{study} a solution provided by the tutor, or \emph{work together} with the tutor to solve the problem. The rewards are obtained after all problems are accomplished, which is defined as the students’ normalized learning gain calculated by the two test scores that students took before and after the experiments~\cite{chi2011empirically}, respectively.
A total of four policies, including three DQN-induced policies (denoted as $\pi_1,\pi_2,\pi_3$) and one expert policy (denoted as $\pi_{expert}$), are deployed to the ITS used by students. 
The log data from students in the prior six semesters are used to train policies and the following semester to test. More details are provided in Appendix~\ref{appendix:pyrenees}.

Figure~\ref{tab:edu-result} shows the results of comparing HOPE (the last row) with the nine original OPE benchmarks (the top section of the left table) using both standard validation methods and signed significance tests for OPE. Overall, HOPE (the last row) outperforms all nine benchmarks in terms of average absolute error (AAE, column 2), regret@1 (column 3), and rank correlation (column 4). There is no clear winner among the nine original OPE benchmarks. 
Column 5 in Table~\ref{tab:edu-result} shows the best policy determined by each OPE. Among them, while all nine original OPE benchmarks select other sub-optimal policies as the best policy, HOPE is the only method that successfully identifies $\pi_1$ to be the best policy in the empirical study. More importantly, while the nine original OPE benchmarks predict that their selected best policy would significantly outperform the Expert policy (column 6), the empirical results (the 7th/last column) show no significant difference was found.  The offline significance test using HOPE, however,  perfectly aligns with the empirical result, that $\pi_1$ is significantly different from the Expert policy in both OPE t-test (column 6) and the Empirical t-test (column 7).   

\subsection{Ablation Studies}
\label{sec:ablation}
For a better understanding of our proposed approach to tackle the partial observability and missing immediate rewards in real-world human-centric environments, we conduct three ablation studies:


\textit{(i) Sparse-HOPE.} One variation of our proposed approach is assuming the preliminary rewards are sparse and calibrating immediate rewards via nearest neighbors, named Sparse-HOPE. Figure~\ref{fig:health-result} and Figure~\ref{tab:edu-result} show that Sparse-HOPE can outperform all benchmarks, except IS, in terms of rank correlation. In real-world intelligent tutors, Sparse-HOPE outperforms all benchmarks in terms of average absolute error and regret@1. Those indicate that our proposed nearest-neighbors-based immediate rewards reconstruction is effective for estimating the return of a policy. On the other hand, Sparse-HOPE performs worse than HOPE, which could indicate the importance of considering rewards as aggregated and the effectiveness of our preliminary rewards reconstruction. 

\textit{(ii) Soft-HOPE.} We define another variation of HOPE, named Soft-HOPE, which assumes decisions made on any observation could contribute equally to the final outcomes, i.e.  
\begin{equation}
     V^{\pi}(\mbox{Soft-HOPE}) = \frac{1}{\sum_{i=1}^N{w_i}} \sum_{i=1}^N{w_i} \sum_{t=1}^{T}{\gamma^{t-1} (\frac{1}{K} \sum_{k}{\mathring{r}^{(k)}_{t'}})}.
\end{equation}
Note that it performs neighbors-based estimation on all observations, as opposed to~\eqref{eqa:npr} which only estimates neighbors-based immediate rewards on critical observations. Figure~\ref{fig:health-result} shows that Soft-HOPE outperforms all benchmarks from real-world medical systems in terms of rank correlation, and from intelligent tutoring in terms of average absolute error. However, it performs worse than most benchmarks in terms of regret@1 and rank correlation from intelligent tutoring. A possible reason is using averaged rewards on all observations could introduce noise to OPE and weaken its estimation of ranking.  Moreover, it performs worse than HOPE from both environments, which indicates that our defined critical observations can help extract meaningful information for OPE.

\textit{(iii) Rand-HOPE.} The third variation of our proposed approach is randomly selecting neighbors instead of using our defined distance for immediate rewards reconstruction, which we call Rand-HOPE. We repeat Rank-HOPE 100 times and report average results. Figure~\ref{fig:health-result} and Figure~\ref{tab:edu-result} show that Rand-HOPE outperforms all benchmarks in both real-world environments. A possible reason is that inferring preliminary immediate rewards can provide much more useful information than sparse rewards, thus even randomly averaged rewards would perform better than using sparse rewards. Rand-HOPE performs worse than HOPE, which indicates that our defined distance is more accurate to reconstruct immediate rewards for OPE in e-learning and healthcare.

Moreover, in real-world intelligent tutoring, HOPE and its three variations, are the only methods that successfully select the best policy. Table~\ref{tab:hope-vs-ablation} further shows the mean and standard error on policies $\pi_1$ and $\pi_{expert}$ estimated by HOPE-related methods. HOPE achieves the best estimation that is closest to ground truth.

\begin{table}[]
    \centering
    \begin{tabular}{ccc}
        \toprule
         & $\pi_1$ Mean$_{\pm se}$ & $\pi_{expert}$ Mean$_{\pm se}$ \\
        \toprule
        Sparse-HOPE & 0.094$_{\pm 0.01}$ & 0.005$_{\pm 0.00}$ \\
        Soft-HOPE & 0.021$_{\pm 0.00}$ & 0.004$_{\pm 0.00}$ \\
        Rand-HOPE & 0.140$_{\pm 0.02}$ & -0.023$_{\pm 0.02}$ \\ 
        HOPE & \textbf{0.176}$_{\pm 0.01}$ & \textbf{0.008}$_{\pm 0.00}$ \\ \midrule
        Empirical result & 0.167$_{\pm 0.02}$ & 0.054$_{\pm 0.06}$ \\
        \bottomrule
    \end{tabular}
    \caption{Mean and standard error with bootstrapping on policies $\pi_1$ and $\pi_{expert}$ from real-world intelligent tutoring. HOPE achieves the best estimation on both policies.}
    \label{tab:hope-vs-ablation}
    \vspace{-15pt}
\end{table}

\section{Conclusion \& Social Impact}
\label{conclusion}
In this work, we proposed an approach, called \textit{HOPE}, for OPE in real-world human-centric environments with partial observability and aggregated rewards. It first inferred preliminary immediate rewards from historical observations. Then it used nearest neighbor methods to fully reconstruct immediate rewards. We also introduced critical observations, that can impact final outcomes of a trajectory over others, to enrich provided information for OPE. We conducted extensive real-world experiments with two challenging tasks for OPE, sepsis treatment, and intelligent tutoring, using both standard validation methods and human-centric significance tests to validate OPE. The results showed that HOPE outperformed state-of-the-art benchmarks in both applications. A part of our methodology leverages WIS, which may introduce variance to estimations. We kept it straightforward such that the performance can be easily isolated. In the future, WIS can be replaced with DR or DICE for reduced variance.

All educational data and EHRs were obtained anonymously through an exempt IRB-approved protocol and were scored using established rubrics. No demographic data or class grades were collected. All data were shared within the research group under IRB, and were de-identified and automatically processed for labeling. This research seeks to remove societal harms that come from lower engagement and retention of students who need more
personalized interventions and developing more robust medical interventions for patients.



\begin{acks}
This research was supported by the NSF Grants: Integrated Data-driven Technologies for Individualized Instruction in STEM Learning Environments (1726550), CAREER: Improving Adaptive Decision Making in Interactive Learning Environments (1651909), and Generalizing Data-Driven Technologies to Improve Individualized STEM Instruction by Intelligent Tutors (2013502). 
\end{acks}



\bibliographystyle{ACM-Reference-Format} 
\bibliography{ope}


\newpage
\clearpage

\appendix

\section{HOPE}

\begin{algorithm}[t]
\SetCustomAlgoRuledWidth{0.5\textwidth} 
\caption{HOPE Algorithm.}\label{alg}
\begin{algorithmic}[1]
\REQUIRE Historical data $\mathcal{D}\sim \beta$. Target policy $\pi$. Parameters $h$ and $K$.
\ENSURE
\STATE Infer $\tilde{r}$ following~\eqref{hope:loss}.
\STATE Initialize an empty set $O^*$. 
\STATE Train $Q^{\beta}(o,a)$ and obtain the set of critical observations $O^*$ given $h$.
\FOR {each $\tau^{(i)}\in \mathcal{D}$}
\FOR {each $o^{(i)}_t\in\tau^{(i)}$}
\IF{$o^{(i)}_t\in O^*$}
\STATE Find $K$ nearest neighbors of $o^{(i)}_t$ following Algorithm~\ref{alg:KNN}.
\STATE Calculate averaged immediate reward $\bar{r}^{(i)}_t$ following~\ref{eqa:knn}.
\ENDIF
\STATE Reconstruct immediate reward $\hat{r}^{(i)}_t$ following~\ref{eq:potential_reward}.
\ENDFOR
\ENDFOR
\end{algorithmic}
\end{algorithm}

In the following, the distance metric for $K$ nearest neighbors can be formulated for both discrete and continuous state spaces. For generalizability, we can use some general distance metrics, such as Kullback–Leibler (KL) divergence~\cite{kullback1951KL}, to calculate $similarity(o^{(i)},o^{(k)})$ and $similarity(a^{(i)},a^{(k)})$. For discrete state spaces $\mathcal{O}=\{o_0,\dots,o_M\}$, 
\begin{equation}
    similarity(o^{(i)},o^{(k)})=\sum_{o_m\in s^{(i)}} \Phi^{(k)}(o_m)\log (\frac{\Phi^{(k)}(o_m)}{\Phi^{(i)}(o_m)}),
\end{equation}
where $\Phi^{(k)}(o_m)$ and $\Phi^{(i)}(o_m)$ are the observation $o_m$ probability distributions on $\tau^{(k)}$ and $\tau^{(i)}$, respectively. Similarly, $similarity(a^{(i)},a^{(k)})=\sum_{a_l\in a^{(i)}} \Phi^{(k)}(a_l)\log (\frac{\Phi^{(k)}(a_l)}{\Phi^{(i)}(a_l)})$ for action space $\mathcal{A}=\{a_1,...,a_L\}$. For continuous observation spaces, we discretize them by characterizing the dependencies across features of observations. Specifically, we group the observations associated with their temporal information (i.e., the elapsed time from the start of the trajectory) into $M$ clusters, where each observation is assigned with one discrete value from the set of clusters. We choose to formulate the objective as defined in the Toeplitz inverse covariance-based clustering problem~\cite{hallac2017toeplitz} with its implementation for multi-series time-aware trajectories~\cite{yang2021mtticc}, since it can be solved through a model-based manner by considering the graphical connectivity.

\section{Experimental Details}
\label{appendix:experiment-detail}

\noindent\textbf{Standard validation metrics}

\textit{Absolute error}
The absolute error is defined as the difference between the actual value and estimated value of a policy:
\begin{equation}
    AE = |V^{\pi}-\hat{V}^{\pi}|
\end{equation}
where $V^{\pi}$ represents the actual value of the policy $\pi$, and $\hat{V}^{\pi}$ represents the estimated value of $\pi$.

\textit{Regret@1}
Regret@1 is the (normalized) difference between the value of the actual best policy, and the actual value of the best policy chosen by estimated values. It can be defined as:
\begin{equation}
    R1 = (\max_{i\in1:P} V^{\pi}_i-\max_{j\in\mbox{best}(1:P)} V^{\pi}_j)/\max_{i\in1:P} V^{\pi}_i
\end{equation}
where $\mbox{best}(1:P)$ denotes the index of the best policy over the set of $P$ policies as measured by estimated values $\hat{V}^{\pi}$.

\textit{Rank correlation}
Rank correlation measures the Spearman's rank correlation coefficient between the ordinal rankings of the estimated values and actual values across policies:
\begin{equation}
    \rho = \frac{Cov(\mbox{rank}(V^{\pi}_{1:P}), \mbox{rank}(\hat{V}^{\pi}_{1:P}))}{\sigma(\mbox{rank}(V^{\pi}_{1:P}) \sigma(\mbox{rank}(\hat{V}^{\pi}_{1:P}))}
\end{equation}
where $\mbox{rank}(V^{\pi}_{1:P})$ denotes the ordinal rankings of the actual values across policies, and $\mbox{rank}(\hat{V}^{\pi}_{1:P})$  denotes the ordinal rankings of the estimated values across policies.

\noindent\textbf{Experimental setup} We implement the proposed method in Python. All experiments are run on a machine with four NVIDIA TITAN Xp / 12GB RAM and two six-core Xeon E5-2620 $@2$ GHz CPUs, CentOS 7.8 (64-bit).

\noindent\textbf{Parameters settings.} In this work, we used $K=5$, following related neighbors-based work~\cite{gottesman2020interpretable}, for generalizability of our work.
For $Q$ used to identify critical observations, we used DQN to learn Q values which contained three hidden layers with 256 nodes on each layer. ReLU activation and 20\% dropout were applied to hidden layers. The output layer followed linear activation.
On each observation, the maximum difference of $Q(o,a),a\in \mathcal{A}$ is calculated. In simulation, since horizon was relative small compared to real-world trajectories, we assumed that each observation can be a critical observation. In real-world environments, the threshold was selected by the elbow method. An example of selecting $h$ in our real-world intelligent tutoring is provided in Figure~\ref{fig:pyrenees-h-thres}.

\begin{figure}
    \centering
    \includegraphics[width=.95\linewidth]{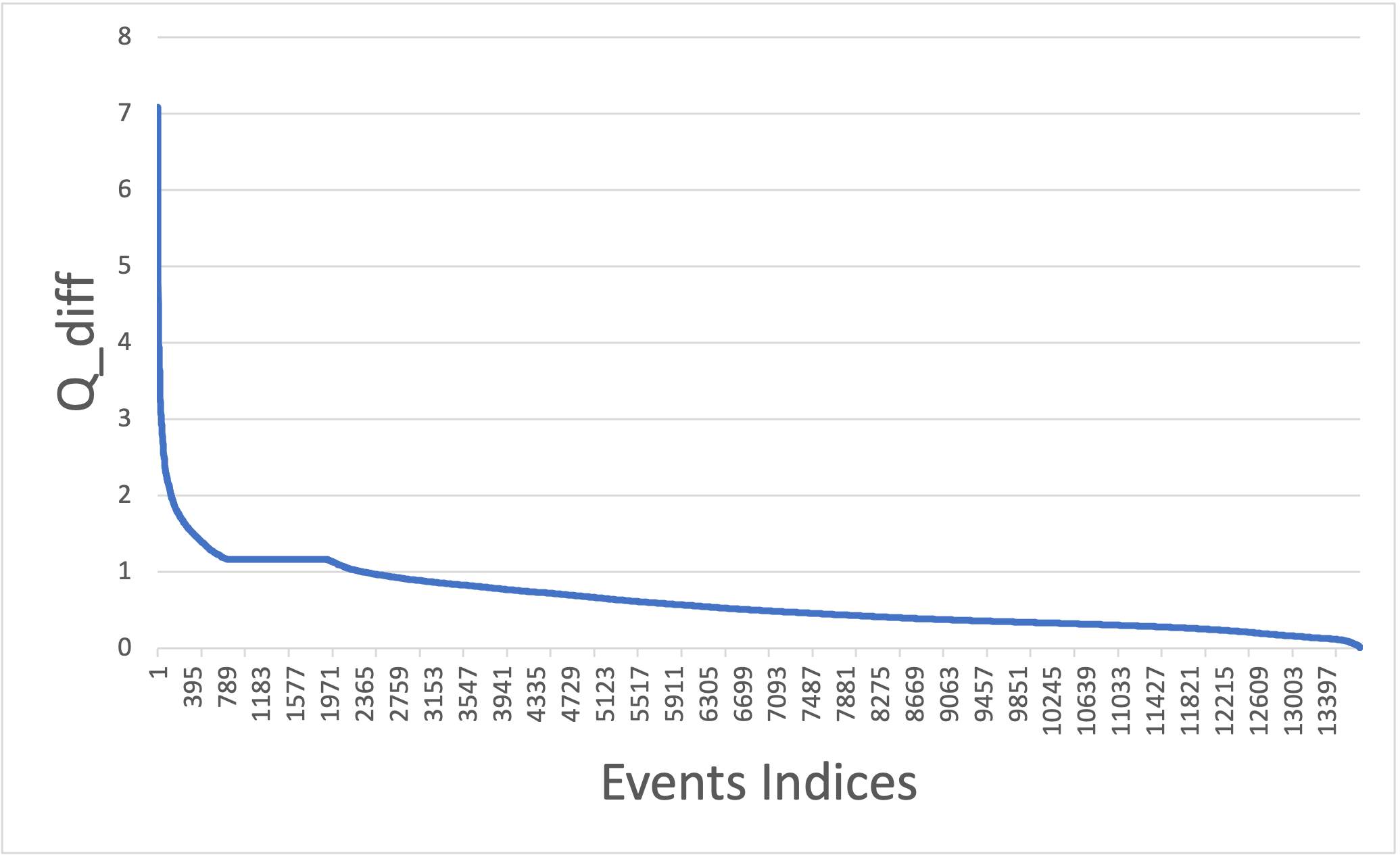}
    \caption{The maximum difference of $Q^{\pi}$ in Intelligent Tutor. We use the elbow method and select $h=1.1$.}
    \label{fig:pyrenees-h-thres}
\end{figure}

\noindent\textbf{Detailed benchmark setting}
For FQE, we used different convergence $\epsilon=\{1e-3,1e-4,1e-5,4e-3\}$ following~\cite{voloshin2019empirical} and report the one with the least average absolute error (or highest rank correlation).
For DualDICE, we used the open-sourced code in its original paper by repeating it three times with different random seeds and report the one with the least average absolute error (or highest rank correlation in real-world medical system).

\subsection{Simulation with ICU Patients}
\label{appendix:sim}

The simulator simulated 10000 trajectories. For each OPE, it estimated each policies, including \texttt{\{optimal\_policy, with\_antibiotics, without\_antibiotics\}} as in~\cite{namkoong2020off}. Here we computed the OPE estimation for evaluation policy 500 times using bootstrap, and calculated the average value of estimation as well as standard deviations.




\subsection{Real-World Medical System}
\label{appendix:health}
\subsubsection{Data Preprocessing}
\noindent\textbf{Labels.} The hospital provided the EHRs over two years, including 221,700 visits with 35 static variables such as gender, age, and past medical condition, and 43 temporal variables including vital signs, lab analytes, and treatments. Our study population is patients with a suspected infection which was identified by the administration of any type of antibiotic, antiviral, antibacterial, antiparasitic, or antifungal, or a positive test result of PCR (Point of Care Rapid). On the basis of the Third International Consensus Definitions for Sepsis and Septic Shock~\cite{singer2016third}, our medical experts identified septic shock as any of the following conditions are met: 
\begin{itemize}
\item Persistent hypertension as shown through two consecutive readings ($\le$ 30 minutes apart).
    \SubItem{Systolic Blood Pressure (SBP) $<$ 90 mmHg}
    \SubItem{Mean Arterial Pressure (MAP) $<$ 65 mmHg}
    \SubItem{Decrease in SBP $\geq$ 40 mmHg with an 8-hour period}
\item Any vasopressor administration. 
\end{itemize}
From the EHRs, 3,499 septic shock positive and 81,398 negative visits were identified based on the intersection of the expert sepsis diagnostic rules and International Codes for Disease 9th division (ICD-9); the 36,122 visits with mismatched labels between the expert rule and the ICD-9 were excluded in our study. 2,205 shock visits were obtained by excluding the visits admitted with septic shock and the long-stay visits and then we did the stratified random sampling from non-shock visits, keeping the same distribution of age, gender, ethnicity, and length of hospital stay. The final data constituted 4,410 visits with an equal ratio of shock and non-shock visits.

\noindent\textbf{Observations.} To approximate patient observations, 15 sepsis-related attributes were selected based on the sepsis diagnostic rules. In our data, the average missing rate across the
15 sepsis-related attributes was 78.6\%. We avoided deleting sparse attributes or resampling with a regular time interval because the attributes suggested by medical experts are critical to decision making for sepsis treatment, and the temporal missing patterns of EHRs also provide the information of patient observations. The missing values were imputed using Temporal Belief Memory~\cite{kim2018temporal} combined with missing indicators~\cite{lipton2016directly}. 

\noindent\textbf{Actions.} For actions, we considered two medical treatments: antibiotic administration and oxygen assistance. Note that the two treatments can be applied simultaneously, which results in a total of four actions. Generally, the treatments are mixed in discrete and continuous action spaces according to their granularity. For example, a decision of whether a certain drug is administrated is discrete, while the dosage of drug is continuous. Continuous action space has been mainly handled by policy-based RL models such as actor-critic models~\cite{lillicrap2015continuous}, and it is generally only available for online RL. Since we cannot search continuous action spaces while online interacting with actual patients, we focus on discrete actions. Moreover, in this work, the RL agent aims to let the physicians know when and which treatment should be given to a patient, rather than suggests an optimal amount of drugs or duration of oxygen control that requires more complex consideration.

\noindent\textbf{Rewards.} Two leading clinicians, both with over 20-year experience on the subject of sepsis, guided to define the reward function based on the severity of septic stages. The rewards were defined as follows: infection [-5], inflammation [-10], organ failures [-20], and septic shock [-50]. Whenever a patient was recovered from any stage of them, the positive reward for the stage was gained back. 

The data was divided into 80\% (the earlier 80\% according to the time of the first event recorded in patients' visits) for training and (the later) 20\% for test, as the most common task for OPE was using historical data to validate policies then applied selected policies for test.

\noindent\textbf{Policies}
We estimate the behavior policy with behavior cloning as in~\cite{fu2021benchmarks,hanna2019importance}.   
The evaluation policies were trained using off-policy DQN algorithm with different learning rates. 

\subsubsection{Septic Shock Rate.}
Since the RL agent cannot directly interact with patients, it only depends on offline data for both policy induction and evaluation.  In similar fashion to prior studies~\cite{komorowski2018artificial,azizsoltani2019unobserved,raghu2017deep}, the induced policies were evaluated using the septic shock rate. The assumption~\cite{raghu2017deep} behind that is: when a septic shock prevention policy is indeed effective, the more the real treatments in a patient trajectory agree with the induced policy, the lower the chance the patient would get into septic shock; vice versa, the less the real treatments in a patient trajectory agree with the induced policy (more dissimilar), the higher the chance the patient would get into septic shock. Specifically, we measured agreement rate with the agent policy, $a\in[0,1]$ was the number of events agreed with the agent policy among the total number of events in a visit; $a=0$ if the actual treatments and the agent’s recommendations are completely different in a visit trajectory, and $a=1$ if they are the same. According to the agreement rate, the average septic shock rate is calculate, which is the number of shock visits among the visits with the corresponding agreement rate $\geq a$. If the agent policies are indeed effective, the more the actually executed treatments agree with the agent policy, the less likely the patient is going to have septic shock. This metric was first used in~\cite{raghu2017deep}.



\subsection{Real-World Intelligent Tutor}
\label{appendix:pyrenees}


\subsubsection{Data Collection and Preprocessing}
Our data contains a total of 1,307 students’ interaction logs with a web-based ITS collected over seven semesters’ classroom studies. During the studies, all students used the same tutor, followed the same general procedure, studied the same training materials, and worked through the same training problems. All students went through the same four phases: 1) reading textbook, 2) pre-test, 3) working on the ITS, and 4) post-test. During reading textbook, students read a general description of each principle, reviewed examples, and solved some training problems to get familiar with the ITS. Then the students took a pre-test which contained a total of 14 single- and multiple-principle problems. Students were not given feedback on their answers, nor were they allowed to go back to earlier questions (so as the post-test). Next, students worked on the ITS, where they received the same 10 problems in the same order. After that, students took the 20-problem post-test, where 14 of the problems were isomorphic to the pre-test and the remainders were non-isomorphic multiple-principle problems. Tests were auto-graded following the same grading criteria. Test scores were normalized to the range of [0, 1].



\noindent\textbf{Rewards.}
There was no immediate reward but the empirical evaluation matrix (i.e., delayed reward), which was the students’ Normalized Learning Gain (NLG). NLG measured students' learning gain irrespective of their incoming competence. NLG is defined as:
$NLG = \frac{score_{posttest}-score_{pretest}}{\sqrt{1-score_{pretest}}}$,
where 1 denotes the maximum score for both pre- and post-test that were taken before and after usage of the ITS, respectively.

\noindent\textbf{Policies.}
The study were conducted across seven semesters, where the first six semesters' data were collected over expert policy and the seventh semester's data were collected over four different policies (three policies were RL-induced policies and one was the expert policy). The expert policy randomly picked actions. The three RL-induced policies were trained using off-policy DQN algorithm with different learning rates.

\end{document}